\documentclass[review]{elsarticle}
\usepackage[utf8]{inputenc}
\usepackage{epstopdf}% To incorporate .eps illustrations using PDFLaTeX, etc.
\usepackage{placeins}
\usepackage{tabularx}
\usepackage[font=small]{subcaption}
\usepackage[numbers]{natbib}
\usepackage{multirow}
\usepackage{algorithm}
\usepackage{algpseudocode}
\usepackage{amsmath}
\usepackage{graphics}
\usepackage{epsfig}
\usepackage{amsthm}
\usepackage{mathrsfs}
\usepackage{amssymb}
\usepackage[font=small,skip=0pt]{caption}
\bibpunct[, ]{(}{)}{;}{a}{}{,}% Citation support using natbib.sty
\theoremstyle{plain}% Theorem-like structures provided by amsthm.sty
\newtheorem{theorem}{Theorem}[section]

\newtheorem{property}[theorem]{Property}

\theoremstyle{definition}
\newtheorem{definition}[theorem]{Definition}

\theoremstyle{remark}

\def \Y {\mathcal{Y}}
\def \Zbb {\mathbb{Z}}

\def \X {\mathcal{X}}
\def \F {\mathcal{F}}

\def \F {\mathcal{F}}

\def \X {\mathcal{X}}

\def \Ac {\mathscr{A}}
\def \Bc {\mathscr{B}}

\def \Zbb {\mathbb{Z}}
\def \Wbb {\mathbb{W}}
\journal{}

\begin{document}

\begin{frontmatter}

\title{Understanding Generalization via Set Theory}%\tnoteref{mytitlenote}}
%\tnotetext[mytitlenote]{Fully documented templates are available in the elsarticle package on \href{http://www.ctan.org/tex-archive/macros/latex/contrib/elsarticle}{CTAN}.}

%% Group authors per affiliation:
%\fnref{myfootnote}
\author[myfootnote]{Shiqi Liu\corref{mycorrespondingauthor}} %\author[mysecondfootnote]{Ting Wang} 
%\author[mysecondfootnote]{Ting chen} \author[mysecondfootnote]{Jintao Lu} \author[mysecondfootnote]{Yu Chen}\author[mysecondfootnote]{Yuxin Yin}
% \author[mysecondfootnote]{Qun Li} \author[mysecondfootnote]{Weiheng Shu}
%\author[mysecondfootnote]{Shenbing Zou}

\address[myfootnote]{shiqi.liu647@foxmail.com, \\Independent Researcher}
%\address[mysecondfootnote]{\{chenting,lujintao,liqun,shuweiheng,wangting\}@diit.cn, \\Innovation Center of Beijing Data Intelligence Information Technology Co., Ltd., Wuhan, Hubei, P.R. China}

%
%%% or include affiliations in footnotes:
\address{}
\cortext[mycorrespondingauthor]{Corresponding author}
%\ead{shenbingzou@gmail.com}
%\author[mysecondaryaddress]{Global Customer Service\corref{mycorrespondingauthor}}

%\address[mymainaddress]{1600 John F Kennedy Boulevard, Philadelphia}
%\address[mysecondaryaddress]{360 Park Avenue South, New York}

\begin{abstract} Generalization is at the core of machine learning models. However, the definition of generalization is not entirely clear. We employ set theory to introduce the concepts of algorithms, hypotheses, and dataset generalization. We analyze the properties of dataset generalization and prove a theorem on surrogate generalization procedures. This theorem leads to our generalization method. Through a generalization experiment on the MNIST dataset, we obtain 13,541 sample bases. When we use the entire training set to evaluate the model's performance, the models achieve an accuracy of 99.945\%. However, if we shift the sample bases or modify the neural network structure, the performance experiences a significant decline. We also identify consistently mispredicted samples and find that they are all challenging examples. The experiments substantiated the accuracy of the generalization definition and the  effectiveness of the proposed methods. Both the set-theoretic deduction and the experiments help us better understand generalization.

\end{abstract}


\end{frontmatter}

%\linenumbers

\section{Introduction}
In recent times, deep learning has achieved significant breakthroughs across various domains (\cite{brown2020language}, \cite{he2017mask}). However, two challenging questions remain: How do deep networks generalize, and why do they generalize?

Many approaches have attempted to address these questions by leveraging learning theory (\cite{vapnik1999overview}) and Bayesian learning theory (\cite{mcallester1998some}). Nevertheless, these approaches tend to focus on the relationship with data distributions and attempt to explain the gap between training and validation set results. 

In a recent study, \cite{li2023transformers} approached in-context learning as an algorithm learning problem with a statistical perspective. \cite{power2022grokking} delved into the learning process of neural networks, suggesting that generalization can occur even beyond what is traditionally considered overfitting. 

We are the first to attempt an understanding of generalization directly through set theory.

In this paper, our focus shifts from the data distribution to the realm of set theory in order to elucidate the process of generalization. We integrate perspectives from ensemble learning and active learning, presenting a novel generalization framework. Our contributions can be summarized as follows:
\begin{enumerate}
  \item We provide a mathematical theory that describes algorithm generalization on the specific dataset using set theory.
  \item We develop a learning framework that leverages ensemble and active learning techniques to maximize the generalization effect..
  \item We implement our methods on the MNIST dataset and identify an efficient generalization set with a generalization effect nearly equal to that of the entire dataset. The experiment also substantiated the accuracy of the generalization definition and the  effectiveness of the proposed methods.
\end{enumerate}
%\begin{figure}
%  \centering
%  \includegraphics[width=.73\textwidth]{experiment/generalization_pic.jpg}
%  \caption{Generalization via Inconsistency}\label{fig:Generalization via Inconsistency}
%\end{figure}

The paper is structured as follows:
\begin{itemize}
  \item Section 2 introduces the mathematical notations and symbols used for describing generalization properties.
  \item Section 3 illustrates properties related to generalization.
  \item Section 4 discusses various aspects of generalization.
  \item Section 5 presents our proposed generalization method.
  \item Section 6 covers the experimental validation of our method.
  \item Finally, in Section 7, we provide our concluding remarks.
\end{itemize}

%We denote scalar, vector, matrix, tensor and corresponding random variable in
%non-bold case, bold case, bold upper case, calligraphic upper case letters.

%\section{Relation}
%Suppose there is a relation $\sim$ defined on domain $\X$ and domain $\Y$. Then it derived two descent relation $\overset{\rightarrow}{\sim}$ and $\overset{\leftarrow}{\sim}$ satisfying that
%
%
%\begin{enumerate}
%  \item [\textbf{Causality}] $\forall x_\alpha\in \X$ and $\forall y_\alpha\in \Y$, if we have 
%$x_\alpha\sim y_\alpha$ then $x_\alpha \overset{\rightarrow}{\sim} y_\alpha$ and $y_\alpha \overset{\leftarrow}{\sim} x_\alpha$.
%  \item [\textbf{R-Closed}] If $\exists y_\beta \in \Y$ satisfying $x_\alpha\sim y_\beta$,  $x_\alpha \overset{\rightarrow}{\sim} y_\beta$ and $y_\beta \overset{\leftarrow}{\sim} x_\alpha$.
%  \item [\textbf{L-Closed}] If $\exists x_\beta \in \Y$ satisfying $x_\beta\sim y_\alpha$,  $x_\beta \overset{\rightarrow}{\sim} y_\alpha$ and $y_\alpha \overset{\leftarrow}{\sim} x_\beta$.
%  \item if \textbf{R-Closed} and \textbf{L-Closed}, then $x_\beta\sim y_\beta$
%  \item 
%\end{enumerate}

\section{Symbology}

\newcommand{\tac}{\begin{aligned}
T_{\Ac}&:&2^\Zbb\rightarrow 2^\F\\
T_{\Ac}(Z)&=&\{f\in F_{\Ac}|\langle f,z\rangle =0 \forall z\in  Z\}
\end{aligned}}

\newcommand{\za}{\begin{aligned}
Z_\Ac=\{z\in\Zbb|\langle f,z\rangle =0 \forall f\in T_{\Ac}(Z)\}
\end{aligned}}

\begin{table}[ht]
\caption{Symbology Illustration\label{tab:Symbology Illustration}}
\small
\centering
\begin{tabularx}{400pt}{ccccc}
\hline
Symbol          & Definition & Formulation&                                                                                                                \\ \hline
$\X$  &total feature alphabet &   \\ \hline
$\Y$ & total class alphabet&    \\\hline

$\F$ & all the mapping from $\X$ to $\Y$ &$\Y^{\X}$   \\\hline
$\Zbb$ & ideal classification mapping & $\Zbb \in \Y^{\X}$ which map each feature to the oracle class   \\\hline
$\Zbb$ & ideal dataset & $\Zbb=\{(x,y)\in\X\times\Y|\Zbb(x)=y, \forall x\in\X\}$   \\\hline
$Z$ & a dataset subset & $Z\subseteq\Zbb$  \\\hline
$\langle f,z\rangle=0$  & $f(x)=y$ $f\in\mathcal{F}$ $z=(x,y)\in\mathbb{Z}$  \\\hline
$F_{\Ac}\subseteq\F$  &  a hypothesis  of algorithm $\Ac$ & $\mathbb{Z}\in F_{\Ac}$  \\\hline
$T_{\Ac}$  &  a set value mapping & $\tac$ \\\hline
\multirow{2}{*}{$\Ac$}  & \multirow{2}{*}{algorithm} &algorithm $\Ac$ maps $Z$ to an element of $T_{\Ac}(Z)$  \\
 & &associating with a $F_{\Ac}$ and $T_{\Ac}$  \\\hline
\multirow{2}{*}{$Z_\Ac$}  & algorithm $\Ac$'s generalization &\multirow{2}{*}{$\za$}  \\
 &  on $Z$  &  \\\hline
\end{tabularx}
\end{table}

The illustration of symbology is in the tabel~\ref{tab:Symbology Illustration}.
\begin{definition} If $\mathbb{Z}\in F_{\Ac}$, $Z_\Ac=\{z\in\Zbb|\langle f,z\rangle =0 \forall f\in T_{\Ac}(Z)\}$ is the algorithm $\Ac$'s generalization on Z. 
  \end{definition}

\section{Property}
\begin{property}\label{pr:t-contain}$V\subseteq W\subseteq \Zbb\Rightarrow T_\Ac(W)\subseteq T_\Ac(V) $
\end{property}

This property shows that a large dataset produce a smaller feasible hypothesis set.

\begin{property}\label{pr:transition}$V\subseteq W\subseteq \Zbb\Rightarrow V_\Ac\subseteq W_\Ac$
\end{property}

This property demonstrates that the generalization set of a large dataset is larger than that of a small dataset.

\begin{property}\label{pr:no-expansion}$V\subseteq\Zbb\Rightarrow V_\Ac=(V_\Ac)_\Ac$
\end{property}

This property indicates that the generalization of the generalization dataset will not expand.

\begin{property}\label{pr:clamp}$V\subseteq W\subseteq \Zbb \quad and\quad V \subseteq W\subseteq V_\Ac\Rightarrow V_\Ac= W_\Ac$
\end{property}

This is the ``squeeze" property of the generalization dataset.

\begin{property}\label{pr:clamp-2}$V,W\subseteq\Zbb \quad and\quad W\subseteq V_\Ac\Rightarrow W_\Ac\subseteq V_\Ac$
\end{property}

This property demonstrates that the containment relationship is preserved during the generalization process.

\begin{property}\label{pr:emptyset}$T_\Ac(\emptyset)=F_\Ac$
\end{property}

This property demonstrates that the empty set will not affect the original hypothesis.

\begin{property}\label{pr:functioncontain}$V\subseteq\Zbb\Rightarrow T_\Ac(V)\subseteq F_\Ac$
\end{property}
This property indicates that any dataset belonging to the oracle dataset will result in a reduction in hypotheses. 

\begin{property}\label{pr:hypothesis subset}$V\subseteq\Zbb \quad and \quad \Zbb  \in F_\Ac \quad and \quad U=\{z\in\X\times\Y|\langle f,z\rangle=0,\forall f\in T_\Ac(V)\}$
\begin{eqnarray}
&\Rightarrow & U\subseteq\Zbb, U=V_\Ac \\
&\Rightarrow & \Zbb -V_\Ac =\Zbb - U \\
&\Rightarrow & W=\{z\in\X\times\Y|\langle f,z\rangle =0 \quad \forall f\in G \subseteq T_\Ac(V)\}\quad \Zbb-W\subseteq \Zbb-V_\Ac
\end{eqnarray}
\end{property}

This property demonstrates that the consistent set on a subset of hypotheses is larger.

\begin{property}\label{pr:generalization condition} $V\subseteq W\subseteq \Zbb$
\begin{eqnarray}
&\Rightarrow& V_\Ac\subseteq W_\Ac\\
&\Rightarrow& if \quad V_\Ac\subset W_\Ac \quad then\quad w\in W_\Ac-V_\Ac \quad and \quad V_\Ac\subset(V\cup\{w\})_\Ac 
\end{eqnarray}
\end{property}

This property illustrates that adding a sample outside the generalization set can monotonically increase generalization.

\begin{property}\label{pr:different algorithm}$F_\Ac\subset F_\Bc$
\begin{eqnarray}
&\Rightarrow& V_\Bc\subseteq V_\Ac\\
&\Rightarrow& if \quad V_\Bc\subset V_\Ac \quad then \quad w\in V_\Ac-V_\Bc \quad and \quad V_\Bc\subset(V\cup\{w\})_\Bc
\end{eqnarray}
\end{property}

This property reveals that a larger number of hypotheses may have smaller generalization sets, and the addition of a sample outside the generalization set can consistently support further generalization.

\begin{theorem}\label{thm:generalziation via inconsistency} Suppose the ideal dataset and mapping are $\Zbb\in F_\Ac$, the selected dataset is $V$, the selected hypothesises $G\subseteq T_\Ac(V)\subseteq F_\Ac$, $W=\{z\in\X\times\Y|\langle f,z\rangle =0 \quad \forall f\in G \}$, then if $w\in \Zbb-W$, we have $V_\Ac \subset (V\cup{w})_\Ac$. 
  
\end{theorem}

This theorem suggests that we can employ a subset of $T_\Ac(V)$ to identify the consistent set and then locate the inconsistent sample. By adding the inconsistent sample to $V$, the generalization of $V\cup\{w\}$ will be  greater than the generalization of $V$.
%
%\begin{property}$F_\Ac\neq\F_\Bc$ and $F_\Ac\cap F_\Bc= F_{\Ac \cap\Bc}$ $\Rightarrow (V_\Ac \cup V_\Bc)_\Ac=(V_\Bc\cup V_\Ac)_\Bc = V_{\Ac\cap\Bc}$.
%\end{property}
%\begin{proof}$V_\Ac=\{z\in\Zbb|\langle f,z\rangle =0 \forall f\in T_\Ac(V)\}$,$V_\Bc=\{z\in\Zbb|\langle f,z\rangle =0 \forall f\in T_\Bc(V)\}$,$(V_\Bc)_\Ac=\{z\in\Zbb|\langle f,z\rangle =0 \forall f\in T_\Ac(V_\Bc)\}$. $T_\Ac(V_\Bc)=\{f\in F_{\Ac}|\langle f,z\rangle =0 \forall z\in  V_\Bc\}=F_\Ac\cap T_\Bc(V)$
%\end{proof}

\begin{property}Construct $f_{\Ac,Z}$ is the $f$ such that $\{\langle f,z\rangle=0 \forall z\in \Zbb\cap Z \quad and\quad \langle f,z\rangle\neq0 \forall z\in \Zbb- Z\}$ then it yields $\forall f \in T_\Ac(Z)$ the number $\vert z\in \Zbb|\langle f,z\rangle=0 \vert\ge\vert z\in \Zbb|\langle f_{(\Ac,Z},z\rangle=0 \vert $, it means that $$accuracy_{f}=\frac{\vert z\in \Zbb|\langle f,z\rangle=0 \vert}{\vert\Zbb\vert}\ge\frac{\vert z\in \Zbb|\langle f_{\Ac,Z},z\rangle=0 \vert} {\vert\Zbb\vert}=\frac{\vert Z\vert}{\vert \Zbb \vert}=accuracy_{f_{\Ac,Z}}.$$
\end{property} 

This property provides an example of the relationship between the accuracy of a hypothesis in $T_\Ac(Z)$ and $f_{\Ac,Z}$.

\section{Discussion of Generalization}
Generalization situation: an algorithm  knows more samples by learning some samples  .

No generalization situation: an algorithm learns samples but fails to know more samples.

\subsection{Why does the algorithm generalize?}

\begin{theorem}$\Ac$ is the algorithm, and $Z$ is dataset. If there is  $f_{Z}\in\F_\Ac$ such that $\langle f_{Z},z\rangle =0 \forall z\in Z$ and $\langle f_{Z},z\rangle \neq 0 \forall z\in \Zbb-Z$, then $\Ac$ will not generalize on $Z$.
\end{theorem}

\begin{proof}
  $f_Z\in T_\Ac(Z)=\{f\in F_\Ac|\langle f,z\rangle =0 \forall z\in Z\}.$
  
  Algorithm $\Ac$'s generalization on $Z$ is $Z_\Ac=\{z \in\Zbb|\langle f,z\rangle=0, f\in T_\Ac(Z) \}\subset \{z \in\Zbb|\langle f_Z,z\rangle=0  \}= Z.$
  
  Therefore there is no generalization for $\Ac$ on $Z$.
\end{proof}

\begin{theorem}$\Ac$ is the algorithm, and $Z$ is dataset. Let $\Zbb-Z=U\cap V$. If there is  $f_{Z\cup U}\in\Ac$ such that $\langle f_{Z\cup U},z\rangle =0 \forall z\in Z\cup U$ and  $\langle f_{Z\cup U},z\rangle \neq 0 \forall z\in V $. And there is  $f_{Z\cup V}\in\Ac$ such that $\langle f_{Z\cup V},z\rangle =0 \forall z\in Z\cup V$ and  $\langle f_{Z\cup V},z\rangle \neq 0 \forall z\i U $ then $\Ac$ will not generalize on $Z$.
\end{theorem}

\begin{proof}
  $f_{Z\cup V},f_{Z\cup U}\in T_\Ac(Z)=\{f\in F_\Ac|\langle f,z\rangle =0 \forall z\in Z\}.$
  
  Algorithm $\Ac$'s generalization on $Z$ is $Z_\Ac=\{z \in\Zbb|\langle f,z\rangle=0 f\in T_\Ac(Z) \}\subset \{z \in\Zbb|\langle f_{Z\cup V},z\rangle=0 \quad and \quad \langle f_{Z\cup U},z\rangle=0 \}= Z.$
  
  Therefore there is no generalization for $\Ac$ on $Z$.
\end{proof}

Algorithms generalize based on the structural or quantitative constraints of the assumed hypotheses. If for all $Z$, exists $f_{Z}\in\Ac$ such that $\langle f_{Z},z\rangle =0 \forall z\in Z$ and $\langle f_{Z},z\rangle \neq 0 \forall z\in Z^C$, then algorithm $\Ac$ will not generalize on any dataset.

As we desire algorithms to possess generalization capabilities, such specific structures within $F_\Ac$ should ideally not occur.

\subsection{Neural network hypothesises}
Suppose the structure  consists of  fixed backbones neural network $$F_\Ac=\{f_w|w\in\Wbb\}.$$ In each round, we train a neural network using gradient descent to adjust its parameters until it achieves 100\% accuracy on the training set. This procedure can be seen as a continuous path algorithm and eliminates all hypotheses that do not belong to $T_{\Ac}(Z)$.

Therefore, we can use neural networks trained from several different initial points to create a surrogate sample of $T_{\Ac}(Z)$. In fact, to obtain $T_{\Ac}(Z)$, we need to iterate through all parameters $w\in \Wbb$ to identify their mappings that satisfy $\langle f_w,z\rangle=0$ for all $z\in Z$.

\section{Methodology}

Building upon theorem~\ref{thm:generalziation via inconsistency}, we can develop an algorithm to generalize via inconsistency by utilizing the set property.

First, we begin with an initial dataset, denoted as $Z_{initial}$, and the entire data pool, denoted as $Z$. We have a total of $n$ neural networks. We set $Z_{start}$ to be equal to $Z_{initial}$.

Next, we train all $n$ neural networks to achieve 100\% accuracy on the dataset $Z_{start}$. This results in the set $\F_n$.

Third, we can calculate the consistent set $Z_{consistent,n,start}=\{z\in\Zbb|\langle f,z\rangle =0 \forall f\in \F_{n}\}$ concerning the $n$ networks.

Fourth, then the generalized ${Z_{start}}_\Ac \subseteq Z_{consistent,n,start}$. Then the generalizing-effective samples belong to ${Z_{start}}_\Ac^c \supseteq Z_{consistent,n,start}^c $. 

Fifth, select one sample $z_{ungeneralized}\in Z_{consistent,n,start}^c\subseteq {Z_{start}}_\Ac^c  $. 

Sixth, construct $Z_{start}=Z_{start}\cup \{z_{ungeneralized}\}$. 

Repeat the steps from second to the sixth until $Z_{consistent,n,start}=Z$.

%The process's illustration is in figure~\ref{fig:Generalization via Inconsistency}.

\section{Experiments and results}
We endeavor to identify efficient bases for a specific dataset using the proposed methods. The term "efficient bases" refers to bases on which the algorithm's generalization can encompass the entire dataset.
\subsection{Datasets}
We utilize the training dataset of MNIST, comprising 60,000 image-label pairs of various handwritten numbers. The specific network structures chosen for the experiments are detailed in the appendix.~\ref{tab:Network Structure}.

\subsection{Quantitative Result}
Based on the experimental data, 10 selected samples from each class have been added to $Z_{initial}$. Additionally, the chosen network structure is L(10,100), which indicates that the fully connected layers have a latent dimension of 100, and the number of classes is 10.

By adding  approximately 13531 samples into $Z_{initial}$, $Z_{13531}$ becomes the bases of 60000 image-label pairs .

This suggests that only 13541 carefully selected sample are needed to enable all 10 models to achieve $100\%$ consistency on the entire dataset. The experiment substantiated the accuracy of the generalization definition and the effectiveness of the proposed methods.

\subsection{Generalization on sample bases shift}

\begin{table}[]
\caption{Generalization on sample bases shift and structure modification}
\label{tab:Generalization on samples basis shift}
\small
\centering
\begin{tabular}{ccccccc}
\hline
Network Structure   
    &Used Data   &Accuracy   \\ 
\hline
L(10,100)  & The selected 13541 bases& $\mathbf{99.94516}\pm 0.00801\%$           \\ \hline
L(10,100)  & The first 13541 samples &  $97.11383\pm 0.17862\%$              \\ \hline
L(10,1000)& The selected 13541 bases & $99.90716\pm 0.02316\%$              \\ \hline
L(10,50)  & The selected 13541 bases &  $99.92100\pm 0.02226\%$         \\ \hline
\end{tabular}
\end{table}

We utilize these 13541 samples to train new models. Their total accuracy on the 60000 samples is 99.945\%(It implies only about 33 samples are mispredicted). However, if we use the first 13541 sample to train new models, the accuracy can only be 97.114\%(It implies about 1732 samples are mispredicted). These results demonstrate the selective 13541 bases have better generalization ability with respect to the same network structure.

\subsection{Generalization on neural nets structure modification}
When we modify the neural network structure to a more complex one, such as $L(10,1000)$, the total accuracy on the dataset of 60,000 samples decreases to 99.90716\%.
When we modify the neural network structure to a simpler one, such as $L(10,50)$, the total accuracy on the dataset of 60,000 samples decreases to 99.9210\%.
This suggests that the choice of basis samples is closely related to the neural network structure.
\subsection{Consistent mispredictions}
\begin{figure}[]
\centering
\captionsetup[subfigure]{font=tiny}
\begin{subfigure}{.13\textwidth}
	\centering
\includegraphics{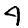}
	\caption{$4\rightarrow 9$}
\end{subfigure}
\hfill
\begin{subfigure}{.13\textwidth}
	\centering
\includegraphics{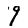}
	\caption{$9\rightarrow 7$}
\end{subfigure}
\hfill
\begin{subfigure}{.13\textwidth}
	\centering
\includegraphics{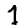}
	\caption{$7\rightarrow 1$}
\end{subfigure}
\hfill
\begin{subfigure}{.13\textwidth}
	\centering
\includegraphics{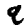}
	\caption{$9\rightarrow 8$}
\end{subfigure}
\hfill
\begin{subfigure}{.13\textwidth}
	\centering
\includegraphics{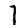}
	\caption{$7\rightarrow 1$}
\end{subfigure}
\hfill
\begin{subfigure}{.13\textwidth}
	\centering
\includegraphics{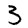}
	\caption{$5\rightarrow 3$}
\end{subfigure}
\hfill
\begin{subfigure}{.13\textwidth}
	\centering
\includegraphics{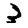}
	\caption{$3\rightarrow 2$}
\end{subfigure}
\hfill
\begin{subfigure}{.13\textwidth}
	\centering
\includegraphics{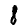}
	\caption{$8\rightarrow 1$}
\end{subfigure}
\hfill
\begin{subfigure}{.13\textwidth}
	\centering
\includegraphics{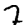}
	\caption{$7\rightarrow 2$}
\end{subfigure}
\hfill
\begin{subfigure}{.13\textwidth}
	\centering
\includegraphics{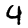}
	\caption{$9\rightarrow 4$}
\end{subfigure}
\hfill
\begin{subfigure}{.13\textwidth}
	\centering
\includegraphics{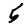}
	\caption{$5\rightarrow 6$}
\end{subfigure}
\hfill\begin{subfigure}{.13\textwidth}
	\centering
\includegraphics{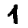}
	\caption{$4\rightarrow 1$}
\end{subfigure}
\hfill
\begin{subfigure}{.13\textwidth}
	\centering
\includegraphics{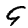}
	\caption{$9\rightarrow 5$}
\end{subfigure}
\hfill
\begin{subfigure}{.13\textwidth}
	\centering
\includegraphics{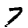}
	\caption{$4\rightarrow 7$}
\end{subfigure}
\hfill
\caption{Consistent but wrong predicted samples.\label{fig:compare consistent but false}The first one is the oracle label, the second one are the consistent predicted label.}
\end{figure}
As shown in figure~\ref{fig:compare consistent but false}, the consistently mispredicted samples are challenging cases. These samples may even lead to differing opinions among individuals regarding their correct labels.
  
\section{Conclusion}
In our study of the generalization properties of algorithms, the generalization step can be seen as a process of hypothesis elimination, ultimately leading to the attainment of a larger consistent set. We've discussed specific cases where an algorithm is unable to generalize effectively. Furthermore, our focus has been on the scenario where the hypotheses are neural networks.

Through experiments conducted on the MNIST dataset, we have observed that only about 13,541 samples forming the basis can enable a neural network to achieve an accuracy of 99.945\% on the training set. However, if we use the first 13,541 samples, the accuracy drops to only 97.114\%. Additionally, modifying the neural network structure leads to a decrease in accuracy. These findings have inspired us to consider more efficient data labeling strategies, which can provide deeper insights into the hypothesis space of neural networks.
The experiment also substantiated the accuracy of the generalization definition and the effectiveness of the proposed methods.

A limitation of our current methods is that they do not consider the topology of the hypothesis space and data space. In future research, we plan to explore the impact of the topology of both the hypothesis and data spaces. Additionally, we aim to extend our study to autoregressive models and other types of models to gain a more comprehensive understanding of generalization.

Furthermore, the discussion of generalization is currently based on the intersection of all feasible hypotheses, which is a strict definition. We recognize that it may be necessary to generalize this definition to encompass more general cases and provide a broader perspective on the concept of generalization.

\subsection{Future Direction}

The generalization theory presented can indeed be extended to large language models. For each valid sample $z$, if it is determined by its $k$ tokens, then we can input it into the large language model with those $k$ tokens. If the model $f$ correctly generates the complete sample using a greedy strategy, then it implies that $\langle f,z\rangle=0$. This suggests that the proposed theory could be applied to address questions related to the generalization of large language models. For instance, in-context learning could be regarded as a selection within the sample space.
 
\section*{Acknowledgement(s)}
 We would like to thank Ting Wang, Ting chen, Jintao Lu, Yu Chen, Yuxin Yin, Qun Li, Weiheng Shu, Shenbing Zou for valuable discussion.

\bibliography{mybibfile}
\appendix

\section{Property Proof}
\noindent\textbf{property~\ref{pr:t-contain}} $V\subseteq W\subseteq \Zbb\Rightarrow T_\Ac(W)\subseteq T_\Ac(V) $

\begin{proof}
  $T_\Ac(W)=\{f\in F_{\Ac}|\langle f,z\rangle =0 \forall z\in  W\} \subseteq \{f\in F_{\Ac}|\langle f,z\rangle =0 \forall z\in  V\}=T_\Ac(V)$
\end{proof}

\noindent\textbf{property~\ref{pr:transition}} $V\subseteq W\subseteq \Zbb\Rightarrow V_\Ac\subseteq W_\Ac$

\begin{proof}
  $V_\Ac=\{z\in\Zbb|\langle f,z\rangle =0 \forall f\in T_{\Ac}(V)\}\subseteq\{z\in\Zbb|\langle f,z\rangle =0 \forall f\in T_{\Ac}(W)=W_\Ac$
\end{proof}

\noindent\textbf{property~\ref{pr:no-expansion}} $V\subseteq\Zbb\Rightarrow V_\Ac=(V_\Ac)_\Ac$

\begin{proof}
$T_\Ac(V_\Ac)=\{f\in F_{\Ac}|\langle f,z\rangle =0 \forall z\in  V_\Ac\}=\{f\in F_{\Ac}|\langle f,z\rangle =0 \forall z\in  V\}$
${V_\Ac}_\Ac=\{z\in\Zbb|\langle f,z\rangle =0 \forall f\in T_\Ac(V_\Ac)\}=\{z\in\Zbb|\langle f,z\rangle =0 \forall f\in T_\Ac(V)\}=V_\Ac.$
\end{proof}

\noindent\textbf{property~\ref{pr:clamp}} $V\subseteq W\subseteq \Zbb \quad and\quad V \subseteq W\subseteq V_\Ac\Rightarrow V_\Ac= W_\Ac$

\begin{proof}
According to property~\ref{pr:no-expansion},
${V_\Ac}_\Ac=V_\Ac$
Therefore, $V \subseteq W\subseteq V_\Ac \Rightarrow V_\Ac=\{z\in\Zbb|\langle f,z\rangle =0 \forall f\in T_\Ac(V_\Ac)\}\subseteq \{z\in\Zbb|\langle f,z\rangle =0 \forall f\in T_\Ac(W)\}= W_\Ac\subseteq \{z\in\Zbb|\langle f,z\rangle =0 \forall f\in T_\Ac(V)\}\Rightarrow W_\Ac=V_\Ac.$
\end{proof}

\noindent\textbf{property~\ref{pr:clamp-2}} $V,W\subseteq\Zbb \quad and\quad W\subseteq V_\Ac\Rightarrow W_\Ac\subseteq V_\Ac$

\begin{proof}
  According to property~\ref{pr:no-expansion} and~\ref{pr:transition},  $W\subseteq V_\Ac\Rightarrow W_\Ac\subseteq {V_\Ac}_\Ac \Rightarrow  W_\Ac\subseteq {V_\Ac}.$
\end{proof}

\noindent\textbf{property~\ref{pr:emptyset}} $T_\Ac(\emptyset)=F_\Ac$

\begin{proof}
 $T_\Ac(\emptyset)=\{f\in F_{\Ac}|\langle f,z\rangle =0 \forall z\in  \emptyset\}=F_\Ac$
\end{proof}

\noindent\textbf{property~\ref{pr:functioncontain}} $V\subseteq\Zbb\Rightarrow T_\Ac(V)\subseteq F_\Ac$

\begin{proof}
  $T_\Ac(V)=\{f\in F_{\Ac}|\langle f,z\rangle =0 \forall z\in  V\}\subseteq F_\Ac$
\end{proof}

\noindent\textbf{property~\ref{pr:hypothesis subset}} $V\subseteq\Zbb \quad and \quad \Zbb  \in F_\Ac \quad and \quad U=\{z\in\X\times\Y|\langle f,z\rangle=0,\forall f\in T_\Ac(V)\}$
\begin{eqnarray}
&\Rightarrow & U\subseteq\Zbb, U=V_\Ac \\
&\Rightarrow & \Zbb -V_\Ac =\Zbb - U \\
&\Rightarrow & W=\{z\in\X\times\Y|\langle f,z\rangle =0 \quad \forall f\in G \subseteq T_\Ac(V)\}\quad \Zbb-W\subseteq \Zbb-V_\Ac
\end{eqnarray}

\begin{proof}
Since  $W\supseteq V_\Ac$, it yields $\Zbb-W\subseteq \Zbb-V_\Ac$.
\end{proof}

\noindent\textbf{property~\ref{pr:generalization condition}} $V\subseteq W\subseteq \Zbb$
\begin{eqnarray}
&\Rightarrow& V_\Ac\subseteq W_\Ac\\
&\Rightarrow& if \quad V_\Ac\subset W_\Ac \quad then\quad w\in W_\Ac-V_\Ac \quad and \quad V_\Ac\subset(V\cup\{w\})_\Ac 
\end{eqnarray}

\begin{proof}
 Since $w\in V\cup\{w\}$, $w\notin V_\Ac$ and $V_\Ac\subseteq (V\cup\{w\})_\Ac$, it yields  $V_\Ac\subset(V\cup\{w\})_\Ac$.
\end{proof}

\noindent\textbf{property~\ref{pr:different algorithm}}$F_\Ac\subset F_\Bc$
\begin{eqnarray}
&\Rightarrow& V_\Bc\subseteq V_\Ac\\
&\Rightarrow& if \quad V_\Bc\subset V_\Ac \quad then \quad w\in V_\Ac-V_\Bc \quad and \quad V_\Bc\subset(V\cup\{w\})_\Bc
\end{eqnarray}

\begin{proof}
   Since $w\in V\cup\{w\}$, $w\notin V_\Bc$ and $V_\Bc\subseteq (V\cup\{w\})_\Bc$, it yields  $V_\Bc\subset(V\cup\{w\})_\Bc$.
\end{proof}

\noindent\textbf{theorem~\ref{thm:generalziation via inconsistency}} Suppose the ideal dataset and mapping are $\Zbb\in F_\Ac$, the selected dataset is $V$, the selected hypothesises $G\subseteq T_\Ac(V)\subseteq F_\Ac$, $W=\{z\in\X\times\Y|\langle f,z\rangle =0 \quad \forall f\in G \}$, then if $w\in \Zbb-W$, we have $V_\Ac \subset (V\cup{w})_\Ac$. 

\begin{proof}
Since  $W\supseteq V_\Ac$, it yields $w\in \Zbb-W \subseteq \Zbb-V_\Ac$. Since $V_\Ac\subseteq (V\cup\{w\})_\Ac$ and $w \notin V_\Ac$, it yields $V_\Ac \subset (V\cup{w})_\Ac$.
\end{proof}

\section{Set value mapping concepts and properties}
\subsection{Set computation}
\begin{property}$T_\Ac(z):=\{f\in F_\Ac |\langle f,z\rangle =0\} \quad and\quad V\subseteq\Zbb\Rightarrow T_\Ac(V)=\cap_{{v\in V}}T_\Ac(v)$
\end{property}

\begin{property}$V,W\subseteq \Zbb \Rightarrow T_\Ac(V\cup W)=T_\Ac(V)\cap T_\Ac(W)$
\end{property}

\begin{property}$V^\alpha\subseteq \Zbb \Rightarrow T_\Ac(\cup_\alpha V^\alpha)=\cap_{\alpha} T_\Ac(V^\alpha)$
\end{property}

\begin{property}$V,W\subseteq \Zbb \Rightarrow T_\Ac(V\cap W)\supseteq T_\Ac(V)\cup T_\Ac(W)$
\end{property}

\begin{property}$V^\alpha\subseteq \Zbb \Rightarrow T_\Ac(\cap_\alpha V^\alpha)\supseteq\cup_{\alpha} T_\Ac(V^\alpha)$
\end{property}

\begin{property}$V\subseteq \Zbb\Rightarrow$ $T_\Ac(V\cup (\Zbb-V))=T_\Ac(V)\cap T_\Ac(\Zbb-V)$
\end{property}

\begin{property}$S_\Ac(z):=\{f\in F_\Ac |\langle f,z\rangle \neq0\} \quad and\quad V\subseteq\Zbb\Rightarrow S_\Ac(V)=\cup_{{v\in V}}S_\Ac(v)$
\end{property}

\begin{property}$V,W\subseteq \Zbb \Rightarrow S_\Ac(V\cap W)\subseteq S_\Ac(V)\cap S_\Ac(W)$
\end{property}

\begin{property}$V^\alpha\subseteq \Zbb \Rightarrow S_\Ac(\cap_\alpha V^\alpha)\subseteq\cap_{\alpha} S_\Ac(V^\alpha)$
\end{property}

\begin{property}$V,W\subseteq \Zbb \Rightarrow S_\Ac(V\cup W)= S_\Ac(V)\cup S_\Ac(W)$
\end{property}

\begin{property}$V^\alpha\subseteq \Zbb \Rightarrow S_\Ac(\cup_\alpha V^\alpha)=\cup_{\alpha} S_\Ac(V^\alpha)$
\end{property}

\begin{property}$V\subseteq \Zbb\Rightarrow$ $S_\Ac(V\cup (\Zbb-V))=S_\Ac(V)\cup S_\Ac(\Zbb-V)$
\end{property}

\begin{property}$V\subseteq \Zbb\Rightarrow$ $S_\Ac(\Zbb-V)\supseteq S_\Ac(\Zbb)- S_\Ac(V)$
\end{property}

\subsection{Inverse of set value mapping}
\begin{definition}[Point function inverse. Correct points predicted by f]
\begin{equation}
T_\Ac^{-1}(f)=\{z\in\Zbb|f\in T_\Ac^{-1}(z)\}
\end{equation}
\end{definition}

\begin{definition}
\begin{equation}
T_{\Ac\supseteq}^{-1}(F) = \{ z\in\Zbb | f\in T_\Ac^{-1}(z) \}
\end{equation}
\end{definition}

\section{Neural network structure}
\begin{table}[]
\caption{Network Structure}
\centering
\label{tab:Network Structure}
\small
\begin{tabular}{ccccccc}
\hline
Dataset& Optimiser & Architecture   
    &  Description  \\ 
\hline
MNIST&Adam&Input  & 28$\times$28$\times$1         \\ 
&$1e-3$&  & Conv $32\times4\times4$ stride $2\times2$ & ReLU activation          \\ 
&&  &  Conv $64\times4\times4$ stride $2\times2$ & ReLU activation          \\
& && FC 100 or 50 or 1000 &ReLU activation                \\ 
&Step 10000&Predictor  & FC 10 &Softmax activation            \\ \hline
\end{tabular}
\end{table}
\end{document}